\documentclass{article}

\usepackage{ijcai15}
\usepackage{times}
\usepackage{amsmath}
\usepackage{amssymb}
\usepackage{amsthm}
\usepackage{multirow}

\usepackage{tikz}
\usepackage{comment}
\usepackage{graphicx}
\newcommand{\tup}[1]{{\langle #1 \rangle}}

\newcommand{\pre}{\mathsf{pre}}     
\newcommand{\eff}{\mathsf{eff}}     
\newcommand{\cond}{\mathsf{cond}}   

\newtheorem{mytheorem}{Theorem}

\title{Hierarchical Finite State Controllers for Generalized Planning \\(Corrected Version)}
\author{
Javier Segovia-Aguas\and Sergio Jim\'enez \and Anders Jonsson\\
Dept.~Information and Communication Technologies, Universitat Pompeu Fabra\\
Roc Boronat 138, 08018 Barcelona, Spain\\
\{javier.segovia,sergio.jimenez,anders.jonsson\}@upf.edu
}

\begin{document}
\maketitle

\begin{abstract}
Finite State Controllers (FSCs) are an effective way to represent sequential plans compactly. By imposing appropriate conditions on transitions, FSCs can also represent generalized plans that solve a range of planning problems from a given domain. In this paper we introduce the concept of {\it hierarchical FSCs} for planning by allowing controllers to call other controllers. We show that hierarchical FSCs can represent generalized plans more compactly than individual FSCs. Moreover, our call mechanism makes it possible to generate hierarchical FSCs in a modular fashion, or even to apply recursion. We also introduce a compilation that enables a classical planner to generate hierarchical FSCs that solve challenging generalized planning problems. The compilation takes as input a set of planning problems from a given domain and outputs a single classical planning problem, whose solution corresponds to a hierarchical FSC. 
\end{abstract}

\section{Introduction}
\label{sec:section1}

Finite state controllers (FSCs) are a compact and effective representation commonly used in AI; prominent examples include robotics~\cite{Brooks89} and video-games~\cite{Buckland04}. In planning, FSCs offer two main benefits: 1) solution compactness~\cite{backstrom2014}; and 2) the ability to represent {\em generalized plans} that solve a range of similar planning problems. This generalization capacity allows FSCs to represent solutions to arbitrarily large problems, as well as problems with partial observability and non-deterministic actions~\cite{Geffner:FSM:AAAI10,Levesque:GPlanning:IJCAI11,Zilberstein:Gplanning:icaps11,Giacomo:FSM:ICAPS13}. 

Even FSCs have limitations, however. Consider the problem of traversing all nodes of a binary tree as in Figure~\ref{fig:traversingbtree}. A classical plan for this task consists of an action sequence whose length is linear in the number of nodes, and hence exponential in the depth of the tree. In contrast, the recursive definition of Depth-First Search (DFS) only requires a few lines of code. However, a standard FSC cannot implement recursion, and the iterative definition of DFS is considerably more complicated, involving an external data structure.

\begin{figure}
  \begin{footnotesize}
    \begin{center}
      \begin{tikzpicture}[level/.style={sibling distance=5cm/#1},scale=.4]        
        \node [circle,draw,scale=.8] {1}
        child {node [circle,draw,scale=.8] {2}
          child {node [circle,draw,scale=.8] {3}}
          child {node [circle,draw,scale=.8] {4}}
        } 
        child {node [circle,draw,scale=.8] {5}
          child {node [circle,draw,scale=.8] {6}}
          child {node [circle,draw,scale=.8] {7}}
        };
      \end{tikzpicture}
    \end{center}
  \end{footnotesize}
  \caption{Example of a binary tree with seven nodes.}
  \label{fig:traversingbtree}
\end{figure}
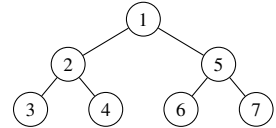

In this paper we introduce a novel formalism for representing and computing compact and generalized planning solutions that we call {\it hierarchical FSCs}. Our formalism extends standard FSCs for planning in three ways. First, a hierarchical FSC can involve multiple individual FSCs. Second, each FSC can call other FSCs. Third, each FSC has a parameter list, and when an FSC is called, it is necessary to specify the arguments assigned to the parameters. As a special case, our formalism makes it possible to implement recursion by allowing an FSC to call itself with different arguments.

To illustrate this idea, Figure~\ref{fig:fsc1} shows an example hierarchical FSC $C[n]$ that implements DFS traversal of binary trees using recursion. Here, $n$ is the lone parameter of the controller and represents the current node of the binary tree. Condition $\mathsf{null}(n)$ tests whether $n$ is points to a null node, while a hyphen `-' indicates that the transition fires no matter what. Action $\mathsf{visit}(n)$ visits node $n$, while $\mathsf{copyL}(n,m)$ and $\mathsf{copyR}(n,m)$ assign the left and right child of node $n$ to $m$, respectively. Action $\mathsf{call}(m)$ is a recursive call to the FSC itself, assigning argument $m$ to the only parameter of the controller and restarting execution from its initial node $Q_0$. 

\begin{figure}[hbt]
  \begin{center}
	\includegraphics[width=0.51\textwidth]{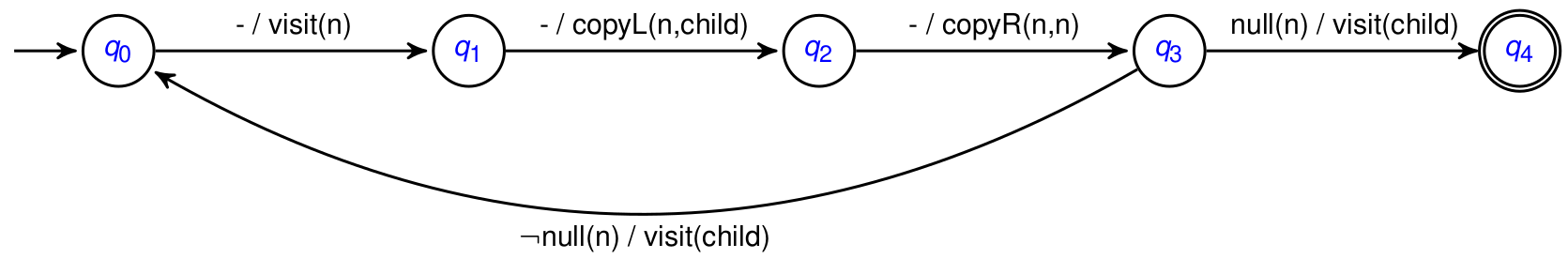}
  \end{center}
  \caption{Hierarchical FSC $C[n]$ that traverses a binary tree.}
  \label{fig:fsc1}
\end{figure}

Intuitively, by repeatedly assigning the right child of $n$ to $n$ itself (using the action $\mathsf{copyR}(n,n)$) and following the cycle of controller states $Q_0,Q_1,Q_2,Q_3,Q_0,\ldots$, the FSC $C[n]$ has the effect of visiting all nodes on the rightmost branch of the tree until a null node is reached. Moreover, by assigning the left child of $n$ to $child$ (using the action $\mathsf{copyL}(n,child)$) and making the recursive call $\mathsf{call}(child)$, the FSC $C[n]$ is recursively executed on all left sub-trees. The controller state $Q_4$ is a terminal state, and the action $\mathsf{visit}(child)$ on the transition to $Q_4$ is in fact not needed and could be removed. However, the FSC is automatically generated by our approach, so we present conditions and actions exactly as they appear.

Compared to previous work on the automatic generation of FSCs for planning the contributions of this paper are:
\begin{enumerate}
\item A reformulation of the transition function of FSCs that allows binary branching only in order to reduce the space of possible controllers.
\item A formal definition of hierarchical FSCs for planning that allows controllers to call other controllers and that includes recursion as a special case.
\item A novel compilation that enables the automatic generation of hierarchical FSCs for challenging generalized planning tasks. The compilation takes as input a set of planning problems from a given domain and outputs a single classical planning problem whose solution corresponds to a hierarchical FSC. This output is expressed in PDDL, thus an off-the-shelf classical planner can be used to generate hierarchical FSCs. The compilation also makes it possible to incorporate prior knowledge in the form of existing FSCs to automatically complete the definition of the remaining FSCs.
\end{enumerate}


\section{Background}
\label{sec:section2}
This section defines our model for classical planning and presents the formalism we use to define FSCs for planning.

\subsection{Classical Planning with Conditional Effects}

We describe states and actions in terms of literals. Formally, given a set of fluents $F$, a literal $l$ is a valuation of a fluent in $F$, i.e.~$l=f$ or $l=\neg f$ for some $f\in F$. A set of literals $L$ thus represents a partial assignment of values to fluents (WLOG we assume that $L$ does not assign conflicting values to any fluent). Given $L$, let $\neg L=\{\neg l:l\in L\}$ be the complement of $L$. A {\em state} $s$ is a set of literals such that $|s|=|F|$, i.e.~a total assignment of values to fluents.

A classical planning problem is a tuple $P=\tup{F,A,I,G}$, where $F$ is a set of fluents, $A$ is a set of actions, $I$ is an initial state and $G$ is a goal condition, i.e.~a set of literals. Each action $a\in A$ has a set of literals $\pre(a)$ called the {\em precondition} and a set of conditional effects $\cond(a)$. Each conditional effect $C\rhd E\in\cond(a)$ is composed of sets of literals $C$ (the condition) and $E$ (the effect). We often describe the initial state $I\subseteq F$ compactly as the subset of fluents that are true.

Action $a$ is applicable in state $s$ if and only if $\pre(a)\subseteq s$, and the resulting set of {\em triggered effects} is
\[
\eff(s,a)=\bigcup_{C\rhd E\in\cond(a),C\subseteq s} E,
\]
i.e.~effects whose conditions hold in $s$. The result of applying $a$ in $s$ is a new state $\theta(s,a)=(s\setminus \neg\eff(s,a))\cup\eff(s,a)$.

A plan for $P$ is an action sequence $\pi=\tup{a_1, \ldots, a_n}$ that induces a state sequence $\tup{s_0, s_1, \ldots, s_n}$ such that $s_0=I$ and, for each $i$ such that $1\leq i\leq n$, $a_i$ is applicable in $s_{i-1}$ and generates the successor state $s_i=\theta(s_{i-1},a_i)$. The plan $\pi$ {\em solves} $P$ if and only if $G\subseteq s_n$, i.e.~if the goal condition is satisfied following the application of $\pi$ in $I$.

\subsection{Finite State Controllers}

Given a planning problem $P=\tup{F,A,I,G}$, an FSC is defined as a tuple $C=\tup{Q,T,q_0,q_\bot}$, where $Q$ is a set of controller states, $T:Q\times 2^F\rightarrow Q\times A$ is a (partial) transition function that assumes full observability, and $q_0\in Q$ and $q_\bot\in Q$ are the initial and terminal controller states, respectively. This definition relates to previous work on FSCs for generalized planning~\cite{Geffner:FSM:AAAI10,Giacomo:FSM:ICAPS13} as follows:
\begin{itemize}
\item Just like in previous approaches (and unlike Mealy machines), transitions do not depend on explicit input sequences but on the current planning state.
\item Previous approaches assume partial observability of the current planning state, defining the transition function $T$ on $Q\times O$, where $O$ is the observation set. We instead define $T$ on $Q\times 2^F$, i.e.~on the full set of fluents.
\item We define an explicit terminal state $q_\bot$, while previous approaches terminate upon reaching the goal condition $G$. The reason is that we will later extend our definition to hierarchies of FSCs where goals $G$ are not necessarily satisfied when the execution of an FSC terminates.
\end{itemize}

We briefly describe the execution semantics of an FSC $C$ on planning problem $P$. The current world state is a pair $(q,s)\in Q\times 2^F$ of a controller state and a planning state. From a pair $(q,s)$, the system transitions to $(q',s')$, where $(q',a)=T(q,s)$ is the result of applying the transition function in $(q,s)$ and $s'=\theta(s,a)$ is the result of applying action $a$ in $s$. Execution starts at $(q_0,I)$ and repeatedly transitions until reaching a pair $(q_\bot,s_\bot)$ that contains the terminal controller state $q_\bot$. An FSC $C$ solves $P$ iff $G\subseteq s_\bot$ upon termination, i.e.~if the goal condition holds in $s_\bot$. The execution of $C$ fails if it reaches a pair $(q,s)$ that was already visited.

A generalized planning problem $\mathcal{P}=\{P_1,\ldots,P_T\}$ is a set of multiple individual planning problems that share fluents and actions. Each individual planning problem $P_t\in\mathcal{P}$ is thus defined as $P_t=\tup{F,A,I_t,G_t}$, where only the initial state $I_t$ and goal condition $G_t$ differ from other planning problems in $\mathcal{P}$. An FSC $C$ solves a generalized planning problem $\mathcal{P}$ if and only if it solves every problem $P_t\in \mathcal{P}$.

\section{Generating Finite State Controllers}
\label{sec:section3}

This section presents a compilation that takes as input a classical planning problem $P=\langle F,A,I,G\rangle$ and a bound $n$ on the maximum number of controller states, and produces as output a classical planning problem $P_n$. Actions in $P_n$ are defined such that any plan that solves $P_n$ has to both generate an FSC $C$ and simulate the execution of $C$ on $P$, thus verifying that $C$ solves $P$. We later extend this compilation to generalized planning problems and hierarchies of FSCs.

To generate an FSC $C=\tup{Q,T,q_0,q_\bot}$ using this compilation we first define $Q=\{q_0,\ldots,q_n\}$ and set $q_\bot\equiv q_n$. The only thing that remains is to construct the transition function $T$. Our approach is to reduce the space of possible controllers by compactly representing $T:Q\times 2^F\rightarrow Q\times A$ using the following three functions $\Gamma$, $\Lambda$ and $\Phi$:
\begin{itemize}
\item $\Gamma:Q\rightarrow F$ associates a fluent $f=\Gamma(q)$ to each $q\in Q$.
\item $\Lambda:Q\times\{0,1\}\rightarrow Q$ returns a successor state in $Q$.
\item $\Phi:Q\times\{0,1\}\rightarrow A$ returns an action in $A$.
\end{itemize}
The transition from a world state $(q,s)$ depends on the truth value of $\Gamma(q)$ in $s$, hence allowing binary branching only. Let $\Gamma(q)\in s$ be a test whose outcome is interpreted as a Boolean value in $\{0,1\}$. The transition function is then defined as $T(q,s)=(\Lambda(q,\Gamma(q)\in s),\Phi(q,\Gamma(q)\in s))$.

We proceed to define $P_n=\{F_n,A_n,I_n,G_n\}$. The idea behind the compilation is to define two types of actions: {\it program actions} that program the three functions $\Gamma$, $\Lambda$ and $\Phi$ for each controller state of $C$, and {\it execute actions} that simulate the execution of $C$ on $P$ by evaluating the functions in the current planning state. 

The set of fluents is $F_n=F\cup F_{T}\cup F_{aux}$, where $F_{T}$ contains the fluents needed to encode the transition function:
\begin{itemize}
\item For each $q\in Q$ and $f\in F$, a fluent $\mathsf{cond}_q^f$ that holds iff $f$ is the condition of $q$, i.e.~if $\Gamma(q)=f$.
\item For each $q,q'\in Q$ and $b\in\{0,1\}$, a fluent $\mathsf{succ}_{q,q'}^b$ that holds iff $\Lambda(q,b)=q'$.
\item For each $q\in Q$, $b\in\{0,1\}$ and $a\in A$, a fluent $\mathsf{act}_{q,a}^b$ that holds iff $\Phi(q,b)=a$.
\item For each $q\in Q$ and $b\in\{0,1\}$, fluents $\mathsf{nocond}_q$, $\mathsf{nosucc}_q^b$ and $\mathsf{noact}_q^b$ that hold iff we have yet to program the functions $\Gamma$, $\Lambda$ and $\Phi$, respectively.
\end{itemize}

Moreover, $F_{aux}$ contains the following fluents:
\begin{itemize}
\item For each $q\in Q$, a fluent $\mathsf{cs}_q$ that holds iff $q$ is the current controller state.
\item Fluents $\mathsf{evl}$ and $\mathsf{app}$ that hold iff we are done evaluating the condition or applying the action corresponding to the current controller state, and fluents $\mathsf{o}^0$ and $\mathsf{o}^1$ representing the outcome of the evaluation.
\end{itemize}

The initial state and goal condition are defined as $I_n=I\cup\{\mathsf{cs}_{q_0}\}\cup\{\mathsf{nocond}_q,\mathsf{noact}_q^b,\mathsf{nosucc}_q^b:q\in Q,b\in\{0,1\}\}$ and $G_n=G\cup\{\mathsf{cs}_{q_n}\}$. Finally, the set of actions $A_n$ replaces the actions in $A$ with the following actions:
\begin{itemize}
\item For each $q\in Q$ and $f\in F$, an action $\mathsf{pcond}_q^f$ for programming $\Gamma(q)=f$:
\begin{align*}
\pre(\mathsf{pcond}_q^f)&=\{\mathsf{cs}_q,\mathsf{nocond}_q\},\\
\eff(\mathsf{pcond}_q^f)&=\{\emptyset\rhd\{\neg\mathsf{nocond}_q,\mathsf{cond}_q^f\}\}.
\end{align*}
\item For each $q\in Q$ and $f\in F$, an action $\mathsf{econd}_q^f$ that evaluates the condition of the current controller state:
\begin{align*}
\pre(\mathsf{econd}_q^f)&=\{\mathsf{cs}_q,\mathsf{cond}_q^f,\neg\mathsf{evl}\},\\
\eff(\mathsf{econd}_q^f)&=\{\emptyset\rhd\{\mathsf{evl}\},\{\neg f\}\rhd\{\mathsf{o^0}\},\{f\}\rhd\{\mathsf{o^1}\}\}.
\end{align*}
\item For each $q\in Q$, $b\in\{0,1\}$ and $a\in A$, an action $\mathsf{pact}_{q,a}^b$ for programming $\Phi(q,b)=a$:
\begin{align*}
\pre(\mathsf{pact}_{q,a}^b)&=\pre(a)\cup\{\mathsf{cs}_q,\mathsf{evl},\mathsf{o}^b,\mathsf{noact}_q^b\},\\
\eff(\mathsf{pact}_{q,a}^b)&=\{\emptyset\rhd\{\neg\mathsf{noact}_q^b,\mathsf{act}_{q,a}^b\}\}.
\end{align*}
\item For each $q\in Q$, $b\in\{0,1\}$ and $a\in A$, an action $\mathsf{eact}_{q,a}^b$ that applies the action of the current controller state:
\begin{align*}
\pre(\mathsf{eact}_{q,a}^b)&=\pre(a)\cup\{\mathsf{cs}_q,\mathsf{evl},\mathsf{o}^b,\mathsf{act}_{q,a}^b,\neg\mathsf{app}\},\\
\eff(\mathsf{eact}_{q,a}^b)&=\eff(a)\cup\{\emptyset\rhd\{\mathsf{app}\}\}.
\end{align*}
\item For each $q,q'\in Q$ and $b\in\{0,1\}$, an action $\mathsf{psucc}_{q,q'}^b$ for programming  $\Lambda(q,b)=q'$:
\begin{align*}
\pre(\mathsf{psucc}_{q,q'}^b)&=\{\mathsf{cs}_q,\mathsf{evl},\mathsf{o}^b,\mathsf{app},\mathsf{nosucc}_q^b\},\\
\eff(\mathsf{psucc}_{q,q'}^b)&=\{\emptyset\rhd\{\neg\mathsf{nosucc}_q^b,\mathsf{succ}_{q,q'}^b\}\}.
\end{align*}
\item For each $q,q'\in Q$ and $b\in\{0,1\}$, an action $\mathsf{esucc}_{q,q'}^b$ that transitions to the next controller state:
\begin{align*}
\pre(\mathsf{esucc}_{q,q'}^b)&=\{\mathsf{cs}_q,\mathsf{evl},\mathsf{o}^b,\mathsf{app},\mathsf{succ}_{q,q'}^b\},\\
\eff(\mathsf{esucc}_{q,q'}^b)&=\{\emptyset\rhd\{\neg\mathsf{cs}_q,\neg\mathsf{evl},\neg\mathsf{o}^b,\neg\mathsf{app},\mathsf{cs}_{q'}\}\}.
\end{align*}
\end{itemize}
Actions $\mathsf{pcond}_q^f$, $\mathsf{pact}_{q,a}^b$ and $\mathsf{psucc}_{q,q'}^b$ program the three functions $\Gamma$, $\Phi$ and $\Lambda$, respectively, while $\mathsf{econd}_q^f$, $\mathsf{eact}_{q,a}^b$ and $\mathsf{esucc}_{q,q'}^b$ execute the corresponding function. Fluents $\mathsf{evl}$ and $\mathsf{app}$ control the order of the execution such that $\Gamma$ is always executed first, then $\Phi$, and finally $\Lambda$.

\begin{mytheorem}\label{thm:fsc}
Any plan $\pi$ that solves $P_n$ induces an FSC $C$ that solves $P$.
\end{mytheorem}

\begin{proof}[Proof sketch]
The only way to change the current controller state is to apply an action of type $\mathsf{esucc}_{q,q'}^b$, which first requires programming and executing the functions $\Gamma$, $\Phi$ and $\Lambda$ in that order. Once programmed, the plan $\pi$ can no longer change these functions since there are no actions that add fluents among $\mathsf{nocond}_q$, $\mathsf{noact}_q^b$ and $\mathsf{nosucc}_q^b$. Once programmed for all states and Boolean values $b\in\{0,1\}$, the three functions $\Gamma$, $\Phi$ and $\Lambda$ together define an FSC $C$.

We show that $\pi$ simulates an execution of $C$ on $P$. The initial state $I\cup\{\mathsf{cs}_{q_o}\}$ corresponds to the world state $(q_0,I)$. In any world state $(q,s)$, the plan has to apply the partial action sequence $\langle\mathsf{econd}_q^f,\mathsf{eact}_{q,a}^b,\mathsf{esucc}_{q,q'}^b\rangle$. Action $\mathsf{econd}_q^f$ adds $\mathsf{o}^b$ where $b\in\{0,1\}$ is the truth value of $f$ in $s$. Action $\mathsf{eact}_{q,a}^b$ applies the action $a$ in $s$ to obtain a new state $s'=\theta(s,a)$. Finally, action $\mathsf{esucc}_{q,q'}^b$ transitions to controller state $q'$. This deterministic execution continues until we reach a terminal state $(q_n,s_n)$ or revisit a world state. If $\pi$ solves $P_n$, execution finishes in $(q_n,s_n)$ and the goal condition $G$ holds in $s_n$, which is the definition of $C$ solving $P$.
\end{proof}

We extend the compilation to address generalized planning problems $\mathcal{P}=\{P_1,\ldots,P_T\}$. In this case a solution to $P_n$ builds an FSC $C$ and simulates the execution of $C$ on all the individual planning problems $P_t\in\mathcal{P}$. The extension introduces actions $\mathsf{end}_t$, $1\leq t<T$, with precondition $G_t\cup\{\mathsf{cs}_{q_n}\}$ and conditional effects that reset the world state to $(q_0,I_{t+1})$ after solving $P_t$. In addition, the initial state and goal condition are redefined as $I_n=I_1\cup\{\mathsf{cs}_{q_0}\}\cup\{\mathsf{nocond}_q,\mathsf{noact}_q^b,\mathsf{nosucc}_q^b:q\in Q,b\in\{0,1\}\}$ and $G_n=G_T\cup\{\mathsf{cs}_{q_n}\}$.

\section{Hierarchical Finite State Controllers}
\label{sec:section4}

This section extends our formalism for FSCs to hierarchical FSCs. We do so by allowing FSCs to call other FSCs. An FSC $C$ can now have parameters, and calls to $C$ specify the arguments passed to the parameters of $C$. Again, we first describe hierarchical FSCs for solving a single planning problem $P=\langle F,A,I,G\rangle$, and then extend the idea to generalized planning.

As in PDDL, we assume that fluents in $F$ are instantiated from predicates. Moreover, we assume that there exist a set of {\em variable objects} $\Omega_v$ and a set of {\em value objects} $\Omega_x$, and that for each $v\in\Omega_v$ and $x\in\Omega_x$, $F$ contains a fluent $\mathsf{assign}_{v,x}$ that models an assignment of type $v=x$. Let $F_a\subseteq F$ be the set of such assignment fluents and let $F_r=F\setminus F_a$ be the remaining fluents.

Given a planning problem $P$ with fluents $F_a\subseteq F$ induced from sets $\Omega_v$ and $\Omega_x$, a hierarchical FSC is a tuple $\mathcal{H}=\langle\mathcal{C},C_1\rangle$, where $\mathcal{C}=\{C_1,\ldots,C_m\}$ is the set of FSCs in the hierarchy and $C_1\in\mathcal{C}$ is the root FSC. We assume that all FSCs in $\mathcal{C}$ share the same set of controller states $Q$ and that each $C_i\in\mathcal{C}$ has an associated parameter list $L_i\in\Omega_v^{k_i}$ consisting of $k_i$ variable objects in $\Omega_v$.  The set of possible {\em FSC calls} is then given by $\mathcal{Z}=\{C_i[p]:C_i\in\mathcal{C},p\in\Omega_v^{k_i}\}$, i.e.~all ways to select an FSC $C_i$ from $\mathcal{C}$ and assign arguments to its parameters. The transition function $T_i$ of each FSC $C_i$ is redefined as $T_i:Q\times 2^F\rightarrow Q\times(A\cup\mathcal{Z})$ to include possible calls to the FSCs in $\mathcal{C}$. As before, we represent $T_i$ compactly using functions $\Gamma_i$, $\Lambda_i$ and $\Phi_i$.

To define the execution semantics of a hierarchical FSC $\mathcal{H}$ we introduce a {\em call stack}. Execution starts in the root FSC, at state $(q_0,I)$ and on level $0$ of the stack. In general, for an FSC $C_i$ and a world state $(q,s)$ and given that $T_i(q,s)=(q',a)$ returns an action $a\in A$, the execution semantics is as explained in Section~\ref{sec:section2} for single FSCs. However, when $T_i(q,s)=(q',C_j[p])$ returns a call to controller $C_j[p]\in\mathcal{Z}$, we set the state on the next level of the stack to $(q_0,s[p])$, where $s[p]$ is obtained from $s$ by copying the value of each variable object in $p$ to the corresponding parameter of $C_j$. Execution then proceeds on the next level of the stack following transition function $T_j$, which can include other FSC calls that invoke higher stack levels. If $T_j$ reaches a terminal state $(q_\bot,s_\bot)$, control is returned to the parent controller $C_i$. Specifically, the state of $C_i$ becomes $(q',s')$, where $s'$ is obtained from $s_\bot$ by substituting the original assignments of values to variables on the previous stack level. The execution of a hierarchical FSC $\mathcal{H}$ terminates when it reaches a terminal state $(q_\bot,s_\bot)$ on stack level $0$, and $\mathcal{H}$ solves $P$ iff $G\subseteq s_\bot$.

\subsection{An Extended Compilation for Hierarchical Finite State Controllers}

We now describe a compilation from $P$ to a classical planning problem $P_{n,m}^\ell=\langle F_{n,m}^\ell,A_{n,m}^\ell,I_{n,m}^\ell,G_{n,m}^\ell\rangle$, such that solving $P_{n,m}^\ell$ amounts to programming a hierarchical FSC $\mathcal{H}=\langle\mathcal{C},C_1\rangle$ and simulating its execution on $P$. As before, $n$ bounds the number of controller states, while $m$ is the maximum number of FSCs in $\mathcal{C}$ and $\ell$ bounds the size of the call stack. The set of fluents is $F_{n,m}^\ell=F_r\cup F_a^\ell\cup F_{T}^m\cup F_{aux}^\ell\cup F_H$ where
\begin{itemize}
\item $F_a^\ell=\{f^l:f\in F_a,0\leq l\leq\ell\}$, i.e.~each fluent of type $\mathsf{assign}_{v,x}$ has a copy for each stack level $l$.
\item $F_{T}^m=\{f^i:f\in F_{T},1\leq i\leq m\}$, i.e.~each fluent in $F_{T}$ has a copy for each FSC $C_i\in\mathcal{C}$ defining its corresponding transition function $T_i$.
\item $F_{aux}^\ell=\{f^l:f\in F_{aux},0\leq l\leq\ell\}$, i.e.~each fluent in $F_{aux}$ has a copy for each stack level $l$.
\end{itemize}
Moreover, $F_H$ contains the following additional fluents:
\begin{itemize}
\item For each $l$, $0\leq l\leq\ell$, a fluent $\mathsf{lvl}^l$ that holds iff $l$ is the current stack level.
\item For each $C_i\in\mathcal{C}$ and $l$, $0\leq l\leq\ell$, a fluent $\mathsf{fsc}^{i,l}$ that holds iff $C_i$ is the FSC being executed on stack level $l$.
\item For each $q\in Q$, $b\in\{0,1\}$, $C_i,C_j\in\mathcal{C}$ and $p\in\Omega_v^{k_j}$, a fluent $\mathsf{call}_{q,j}^{b,i}(p)$ that holds iff $\Phi_i(q,b)=C_j[p]$.
\end{itemize}
The initial state and goal condition are now defined as $I_{n,m}^\ell=(I\cap F_r)\cup\{f^0:f\in I\cap F_a\}\cup\{\mathsf{cs}_{q_0}^0,\mathsf{lvl}^0,\mathsf{fsc}^{1,0}\}\cup\{\mathsf{nocond}_q^i,\mathsf{noact}_q^{b,i},\mathsf{nosucc}_q^{b,i}:q\in Q,b\in\{0,1\},C_i\in\mathcal{C}\}$ and $G_{n,m}^\ell=G\cup\{\mathsf{cs}_{q_n}^0\}$. In other words, fluents of type $\mathsf{assign}_{v,x}\in F_a$ are initially marked with stack level $0$, the controller state on level $0$ is $q_0$, the current stack level is $0$, the FSC on level $0$ is $C_1$, and functions $\Gamma_i$, $\Lambda_i$ and $\Phi_i$ are yet to be programmed for any FSC $C_i\in\mathcal{C}$. To satisfy the goal we have to reach the terminal state $q_n$ on level $0$ of the stack.

To establish the actions in the set $A_{n,m}^\ell$, we first adapt all actions in $A_n$ by parameterizing on the FSC $C_i\in\mathcal{C}$ and stack level $l$, $0\leq l\leq\ell$, adding preconditions $\mathsf{lvl}^l$ and $\mathsf{fsc}^{i,l}$, and modifying the remaining preconditions and effects accordingly. As an illustration we provide the definition of the resulting action $\mathsf{pcond}_q^{f,i,l}$:
\begin{align*}
\pre(\mathsf{pcond}_q^{f,i,l})&=\{\mathsf{lvl}^l,\mathsf{fsc}^{i,l},\mathsf{cs}_q^l,\mathsf{nocond}_q^i\},\\
\eff(\mathsf{pcond}_q^{f,i,l})&=\{\emptyset\rhd\{\neg\mathsf{nocond}_q^i,\mathsf{cond}_q^{f,i}\}\}.
\end{align*}
Compared to the old version of $\mathsf{pcond}_q^f$, the current controller state $\mathsf{cs}_q^l\in F_{aux}^\ell$ refers to the stack level $l$, and fluents $\mathsf{nocond}_q^i$ and $\mathsf{cond}_q^{f,i}$ in $F_{T}^m$ refer to the FSC $C_i$. The precondition models the fact that we can only program the function $\Gamma_i$ of $C_i$ in controller state $q$ on stack level $l$ when $l$ is the current stack level, $C_i$ is being executed on level $l$, the current controller state on level $l$ is $q$, and $\Gamma_i$ has not been previously programmed in $q$.

In addition to the actions adapted from $A_n$, the set $A_{n,m}^\ell$ also contains the following new actions:
\begin{itemize}
\item For each $q\in Q$, $b\in\{0,1\}$, $C_i,C_j\in\mathcal{C}$, $p\in\Omega_v^{k_j}$ and $l$, $0\leq l<\ell$, an action $\mathsf{pcall}_{q,j}^{b,i,l}(p)$ to program a call from the current FSC, $C_i$, to FSC $C_j$:
\begin{align*}
\pre(\mathsf{pcall}_{q,j}^{b,i,l}(p))&=\{\mathsf{lvl}^l,\mathsf{fsc}^{i,l},\mathsf{cs}_q^l,\mathsf{evl}^l,\mathsf{o}^{b,l},\mathsf{noact}_q^{b,i}\},\\
\eff(\mathsf{pcall}_{q,j}^{b,i,l}(p))&=\{\emptyset\rhd\{\neg\mathsf{noact}_q^{b,i},\mathsf{call}_{q,j}^{b,i}(p)\}\}.
\end{align*}
\item For each $q\in Q$, $b\in\{0,1\}$, $C_i,C_j\in\mathcal{C}$, $p\in\Omega_v^{k_j}$ and $l$, $0\leq l<\ell$, an action $\mathsf{ecall}_{q,j}^{b,i,l}(p)$ that executes a call:
\begin{align*}
\pre&(\mathsf{ecall}_{q,j}^{b,i,l}(p))=\\&\{\mathsf{lvl}^l,\mathsf{fsc}^{i,l},\mathsf{cs}_q^l,\mathsf{evl}^l,\mathsf{o}^{b,l},\mathsf{call}_{q,j}^{b,i}(p),\neg\mathsf{app}^l\},\\
\eff&(\mathsf{ecall}_{q,j}^{b,i,l}(p))=\{\emptyset\rhd\{\neg\mathsf{lvl}^l,\mathsf{lvl}^{l+1},\mathsf{cs}_{q_0}^{l+1},\mathsf{app}^l\}\}\\
\cup\;&\{\{\mathsf{assign}_{p^k,x}^l\}\rhd\{\mathsf{assign}_{L_j^k,x}^{l+1}\}:1\leq k\leq k_j,x\in\Omega_x\}.
\end{align*}
\item For each $C_i\in\mathcal{C}$ and $l$, $0<l\leq\ell$, an action $\mathsf{term}^{i,l}$:
\begin{align*}
\pre(\mathsf{term}^{i,l})&=\{\mathsf{lvl}^l,\mathsf{fsc}^{i,l},\mathsf{cs}_{q_n}^l\},\\
\eff(\mathsf{term}^{i,l})&=\{\emptyset\rhd\{\neg\mathsf{lvl}^l,\neg\mathsf{fsc}^{i,l},\neg\mathsf{cs}_{q_n}^l,\mathsf{lvl}^{l-1}\}\}\\
&\cup\;\{\emptyset\rhd\{\neg\mathsf{assign}_{v,x}^l:v\in\Omega_v,x\in\Omega_x\}\}.
\end{align*}
\end{itemize}
As an alternative to $\mathsf{pact}_{q,a}^{b,i,l}$, the action $\mathsf{pcall}_{q,j}^{b,i,l}(p)$ programs an FSC call $C_j[p]$, i.e.~defines the function as $\Phi_i(q,b)=C_j[p]$. Action $\mathsf{ecall}_{q,j}^{b,i,l}(p)$ executes this FSC call by incrementing the current stack level to $l+1$ and setting the controller state on level $l+1$ to $q_0$. The conditional effect $\{\mathsf{assign}_{p^k,x}^l\}\rhd\{\mathsf{assign}_{L_j^k,x}^{l+1}\}$ effectively copies the value of the argument $p^k$ on level $l$ to the corresponding parameter $L_j^k$ of $C_j$ on level $l+1$. When in the terminal state $q_n$, the termination action $\mathsf{term}^{i,l}$ decrements the stack level to $l-1$ and deletes all temporary information about stack level $l$.

\begin{mytheorem}\label{thm:hfsc}
Any plan $\pi$ that solves $P_{n,m}^\ell$ induces a hierarchical FSC $\mathcal{H}$ that solves $P$.
\end{mytheorem}

\begin{proof}[Proof sketch]
Similar to the argument in the proof of Theorem~\ref{thm:fsc}, the plan $\pi$ has to program the functions $\Gamma_i$, $\Lambda_i$ and $\Phi_i$ of each FSC $C_i\in\mathcal{C}$. Because of the new actions $\mathsf{pcall}_{q,j}^{b,i,l}(p)$, this includes the possibility of making FSC calls. Hence $\pi$ implicitly defines a hierarchical FSC $\mathcal{H}$.

Moreover, $\pi$ simulates an execution of $\mathcal{H}$ on $P$ starting from $(q_0,I)$ on stack level $0$. In any world state $(q,s)$ on stack level $l$ while executing the FSC $C_i$, whenever the plan contains a partial action sequence $\langle\mathsf{econd}_q^{f,i,l},\mathsf{ecall}_{q,j}^{b,i,l}(p),\mathsf{esucc}_{q,q'}^{b,i,l}\rangle$ that involves an FSC call, the effect of $\mathsf{ecall}_{q,j}^{b,i,l}(p)$ is to increment the stack level, causing execution to proceed on stack level $l+1$ for the FSC $C_j$. The only action that decrements the stack level is $\mathsf{term}^{j,{l+1}}$, which is only applicable once we reach the terminal state $q_n$ on stack level $l+1$. Once $\mathsf{term}^{j,{l+1}}$ has been applied, we can now apply action $\mathsf{esucc}_{q,q'}^{b,i,l}$ to transition to the new controller state $q'$.

 If $\pi$ solves $P_{n,m}^\ell$, execution terminates in a state $(q_n,s_n)$ on level $0$ and the goal condition holds in $s_n$, satisfying the condition that $\mathcal{H}$ solves $P$.
\end{proof}

\begin{scriptsize}
\begin{table*}[hbt!]
\centering
\begin{tabular}{l@{\hspace*{5pt}}r@{\hspace*{5pt}}r@{\hspace*{5pt}}r@{\hspace*{5pt}}r@{\hspace*{5pt}}r@{\hspace*{5pt}}r@{\hspace*{5pt}}r}
 \textbf{Domain} & \textbf{Controllers} & \textbf{Solution} & \textbf{States} & \textbf{Instances} & \textbf{Time(s)} & \textbf{Total time (s)} & \textbf{Plan length} \\\hline
Blocks			&	1			&	OC			&	3			&	5				&	2				&	2	&	64 \\
Gripper			&	1			&	OC			&	3			&	2				&	12				&	12 &	111 \\
List			&	1			&	OC			&	2			&	6				&	0.23			&	0.23 &	158\\
Reverse			&	1			&	OC			&	3			&	2				&	64				&	64	&	61\\
Summatory		&	1			&	OC			&	2			&	4				&	8				&	8 &	60 \\
Tree/DFS		&	1			&	RP			&	4			&	1				&	141				&	141	&	102 \\
Visitall		&	3			&	HC			&	2, 1, 2		        &	3, 3, 3			        &	1, 2, 1		        &	4	& 83, 74, 297 \\
\end{tabular}
\caption{Number of controllers used, solution kind (OC=One Controller, HC=Hierarchical Controller, RP=Recursivity with Parameters) and, for each controller: number of states, number of instances in $\mathcal{P}$, planning time and plan length.}
\label{tab:results}
\end{table*}
\end{scriptsize}

We remark that the action $\mathsf{pcall}_{q,j}^{b,i,l}(p)$ can be used to implement recursion by setting $i\equiv j$, making the FSC $C_i$ call itself. We can also partially specify the functions $\Gamma_i$, $\Lambda_i$ and $\Phi_i$ of an FSC $C_i$ by adding fluents of type $\mathsf{cond}_q^{f,i}$, $\mathsf{act}_{q,a}^{b,i}$, $\mathsf{succ}_{q,q'}^{b,i}$ and $\mathsf{call}_{q,j}^{b,i}(p)$ to the initial state $I_{n,m}^\ell$. This way we can incorporate prior knowledge regarding the configuration of some previously existing FSCs in $\mathcal{C}$.

The compilation can be extended to a generalized planning problem $\mathcal{P}=\{P_1,\ldots,P_T\}$ in a way analogous to $P_n$. Specifically, each action $\mathsf{end}_t$, $1\leq t<T$, should have precondition $G_t\cup\{\mathsf{cs}_{q_n}^0\}$ and reset the state to $I_{t+1}\cup\{\mathsf{cs}_{q_0}^0\}$, i.e.~the system should reach the terminal state $q_n$ on stack level $0$ and satisfy the goal condition $G_t$ of $P_t$ before execution proceeds on the next problem $P_{t+1}\in\mathcal{P}$. To solve $P_{n,m}^\ell$, a plan hence has to simulate the execution of $\mathcal{H}$ on all planning problems in $\mathcal{P}$.

\section{Evaluation}
\label{sec:section6}
We evaluate our approach in a set of generalized planning benchmarks and programming tasks taken from \citeauthor{Geffner:FSM:AAAI10}~[\citeyear{Geffner:FSM:AAAI10}] and \citeauthor{javi-Gplanning-ICAPS16}~[\citeyear{javi-Gplanning-ICAPS16}]. In all experiments, we run the classical planner Fast Downward~\cite{Helmert:FD:JAIR06} with the {\sc Lama-2011} setting~\cite{richter:lama:JAIR2010} on a processor Intel Core i5 3.10GHz x 4 with a 4GB memory bound and time limit of 3600s.

We briefly describe each domain used in experiments. In Blocks, the goal is to unstack blocks from a single tower until a green block is found. In Gripper, the goal is to transport a set of balls from one room to another. In List, the goal is to visit all the nodes of a linked list. In Reverse, the goal is to reverse the elements of a list. In Summatory, the goal is to compute the sum $\sum_i^ni$ for a given input $n$. In Tree/DFS, the goal is to visit all nodes of a binary tree. Finally, in Visitall, the goal is to visit all the cells of a square grid.

Table~\ref{tab:results} summarizes the obtained experimental results. In all but two domains our compilation makes it possible to find a single FSC (OC=One Controller) that solves all planning instances in the input. Moreover, we manually verified that the resulting FSC solves all other instances from the same domain. These results reflect those of earlier approaches, but in the domains from \citeauthor{javi-Gplanning-ICAPS16}~[\citeyear{javi-Gplanning-ICAPS16}], the FSC is able to store generalized plans more compactly, and generation of the FSC is faster.

In Tree/DFS, as mentioned in the introduction, generating a single FSC that solves the problem iteratively without recursive calls is difficult. In contrast, since our compilation simulates a call stack, we are able to automatically generate the FSC in Figure~\ref{fig:fsc1}. There are some discrepancies with respect to the compilation that we address below:
\begin{itemize}
\item As described, a solution to the compiled planning problem $P_{n,m}^\ell$ has to program a condition for each controller state, while the FSC in Figure~\ref{fig:fsc1} includes deterministic transitions. However, since all fluents in $f$ are potential conditions, programming a condition on a fluent that is static is effectively equivalent to programming a deterministic transition, since the outcome of the evaluation will always be the same for this fluent.
\item In the solution generated by the planner, the condition $\mathsf{null}(n)$ is actually emulated by a condition $\mathsf{equals}(n,n)$, where $\mathsf{equals}$ is a derived predicate that tests whether two variables have the same value. The reason this works is that when applied to a leaf node $n$, the action $\mathsf{copyR}(n,n)$ deletes the current value of $n$ without adding another value, since $n$ does not have a right child. Hence evaluating $\mathsf{equals}(n,n)$ returns false, since there is no current value of $n$ to unify over.
\item As previously mentioned, the transition to the terminal state $Q_4$ includes an action $\mathsf{visit}(child)$ which is not needed; the reason this action is generated by the planner is that there is no option for leaving the action ``blank''. Effectively, when executing the FSC the action in question has no effect.
\end{itemize}

Finally, in Visitall, attempting to generate a single controller for solving all input instances fails. Moreover, even if we set $m>1$ and attempt to generate a hierarchical controller from scratch, the planner does not find a solution within the given time bound. Instead, our approach is to generate a hierarchical FSC incrementally. We first generate two single FSCs, where the first solves the subproblem of iterating over a single row, visiting all cells along the way, and the second solves the subproblem of returning to the first column. We then use the compilation to generate a planning problem $P_{n,m}^\ell$ in which two of the FSCs are already programmed, so the classical plan only has to automatically generate the root controller.

\section{Related Work}
\label{sec:section7}

The main difference with previous work on the automatic generation of FSCs~\cite{Geffner:FSM:AAAI10,Giacomo:FSM:ICAPS13} is that they generate single FSCs relying on a partially observable planning model. In contrast, our compilation generate hierarchical FSCs that can branch on any fluent since we consider all fluents as observable. Our approach also makes it possible to generate recursive slutions and to incorporate prior knowledge as existing FSCs, and automatically complete the definition of the remaining hierarchical FSC.

Hierarchical FSCs are similar to {\em planning programs}~\cite{Jimenez15,javi-Gplanning-ICAPS16}. Programs are a special case of FSCs, and in general, FSCs can represent a plan more compactly. Another related formalism is {\em automaton plans}~\cite{backstrom2014}, which also store sequential plans compactly using hierarchies of finite state automata. However, automaton plans are Mealy machines whose transitions depend on the symbols of an explicit input string. Hence automaton plans cannot store generalized plans, and their focus is instead on the compression of sequential plans.

FSCs can also represent other objects in planning. \citeauthor{Hickmott07}~[\citeyear{Hickmott07}] and \citeauthor{LaValle06}~[\citeyear{LaValle06}] used FSCs to represent the entire planning instance. In contrast, \citeauthor{Toropila10}~[\citeyear{Toropila10}] used FSCs to represent the domains of individual variables of the instance. \citeauthor{Baier:McIlraith:icaps06}~[\citeyear{Baier:McIlraith:icaps06}] showed how to convert an LTL representation of \emph{temporally extended goals}, i.e.~conditions that must hold over the intermediate states of a plan, into a non-deterministic FSC.

\section{Conclusion}
\label{sec:section8}

In this paper we have presented a novel formalism for hierarchical FSCs in planning in which controllers can recursively call themselves or other controllers to represent generalized plans more compactly. We have also introduced a compilation into classical planning which makes it possible to use an off-the-shelf planner to generate hierarchical FSCs. Finally we have showed that hierarchical FSCs can be generated in an incremental fashion to address more challenging generalized planning problems.


Just as in previous work on the automatic generation of FSCs, our compilation takes as input a bound on the number of controller states. Furthermore, for hierarchical FSCs we specify bounds on the number of FSCs and stack levels. An iterative deepening approach could be implemented to automatically derive these bounds. Another issue is the specification of representative subproblems to generate hierarchical FSCs in an incremental fashion. Inspired by ``Test Driven Development''~\cite{beck2001agile}, we believe that defining subproblems is a step towards automation.

Last but not least, we follow an inductive approach to generalization, and hence we can only guarantee that the solution generalizes over the instances of the generalized planning problem, much as in previous work on computing FSCs. With this said, all the controllers we report in the paper generalize. In machine learning, the validation of a generalized solution is traditionally done by means of statistics and validation sets. In planning this is an open issue, as well as the generation of relevant examples that lead to solutions that generalize.

\subsection*{Acknowledgment}
\begin{small}
This work is partially supported by grant TIN2015-67959 and the Maria de Maeztu Units of Excellence Programme MDM-2015-0502, MEC, Spain.
Sergio Jim\'enez is partially supported by the {\it Juan de la Cierva} program funded by the Spanish government.
\end{small}

\bibliographystyle{named}
\bibliography{paper-ijcai-fsc-2016.bib}

\end{document}